\documentclass[article]{IEEEtran}
\usepackage[top=1in, bottom=1in, left=0.75in, right=0.75in]{geometry}

\usepackage{amsmath, amssymb, epsfig, graphicx, bm}
\usepackage{cases}
\usepackage{float, subfigure}
\usepackage{color}
\usepackage{amsthm}
\usepackage{amsfonts}
\usepackage{cite, url}
\usepackage[all,cmtip]{xy}
\usepackage{color,hyperref}
\usepackage{algorithmic,algorithm}

\usepackage{amsmath}
\usepackage{amsthm}
\usepackage{amsfonts}
\usepackage{algorithm}
\usepackage{algorithmic}

\newtheorem{Lemma}{Lemma}
\newtheorem{Theorem}{Theorem}
\newtheorem{Assumption}{Assumption}
\newtheorem{Remark}{Remark}

\newcommand{\avg}[2]{\frac 1 {#2}\sum_{{#1}=1}^{#2}}
\newcommand{\avgHonest}{\frac 1 {W-B}\sum_{w\notin \mathcal{B}}}

\begin{document}

\title{Federated Variance-Reduced Stochastic Gradient Descent with Robustness to Byzantine Attacks}
\author{\authorblockN{Zhaoxian Wu, Qing Ling, Tianyi Chen, and Georgios B. Giannakis}

\thanks{Zhaoxian Wu and Qing Ling are with School of Data and Computer Science and Guangdong Province Key Laboratory of Computational Science, Sun Yat-Sen University, Guangzhou, Guangdong 510006, China. Tianyi Chen is with Department of Electrical, Computer, and Systems Engineering, Rensselaer Polytechnic Institute, Troy, New York 12180, USA. Georgios B. Giannakis is with Department of Electrical and Computer Engineering and Digital Technology Center, University of Minnesota, Minneapolis, Minnesota 55455, USA. Qing Ling is supported in part by NSF China Grants 61573331 and 61973324, and Fundamental Research Funds for the Central Universities. Georgios B. Giannakis is supported in part by NSF Grants 1509040, 1508993, 1711471, and 1901134. A short version of this paper has been submitted to IEEE International Conference on Acoustics, Speech, and Signal Processing, Barcelona, Spain, May 4-8, 2020. Corresponding Email: lingqing556@mail.sysu.edu.cn.}}

\maketitle

\begin{abstract}
This paper deals with distributed finite-sum optimization for learning over networks in the presence of malicious Byzantine attacks. To cope with such attacks, most resilient approaches so far combine stochastic gradient descent (SGD) with different robust aggregation rules. However, the sizeable SGD-induced stochastic gradient noise makes it challenging to distinguish malicious messages sent by the Byzantine attackers from noisy stochastic gradients sent by the `honest' workers. This motivates us to reduce the variance of stochastic gradients as a means of robustifying SGD in the presence of Byzantine attacks. To this end, the present work puts forth a Byzantine attack resilient distributed (Byrd-) SAGA approach for learning tasks involving finite-sum optimization over networks. Rather than the mean employed by distributed SAGA, the novel Byrd-SAGA relies on the geometric median to aggregate the corrected stochastic gradients sent by the workers. When less than half of the workers are Byzantine attackers, the robustness of geometric median to outliers enables Byrd-SAGA to attain provably linear convergence to a neighborhood of the optimal solution, with the asymptotic learning error determined by the number of Byzantine workers. Numerical tests corroborate the robustness to various Byzantine attacks, as well as the merits of Byrd-SAGA over Byzantine attack resilient distributed SGD.
\end{abstract}




\section{Introduction}
\label{sec:introduction}

With the rapid development of information technologies, the volume of distributed data increases explosively. Every day, numerous distributed devices including sensors, cellphones, computers, and vehicles, generate huge amounts of data, which are often forwarded to datacenters for further processing and learning tasks. However, collecting data from distributed devices and storing them in datacenters raise major privacy concerns \cite{Srikant2000,Duchi2013,Zhou2018}. Accounting for these concerns, federated learning has been advocated to provide a privacy-preserving, decentralized data processing and machine learning framework \cite{Konecny2016}. Data in federated learning are kept private, and local computations are carried at the distributed devices. Updates of local variables (such as stochastic gradients, corrected stochastic gradients, and model parameters) are found using per-device private data, while the datacenter aggregates local variables and disseminates the aggregated result to the distributed devices.

Even though privacy is preserved, the distributed nature of federated learning makes it vulnerable to errors and adversarial attacks. Devices can then become unreliable in either computing or communicating, or, they can even be hacked by adversaries. As a result, compromised devices may send malicious messages to the datacenter, thus misleading the learning process \cite{Varshney2013-review,Moura2018-review,Bajwa2019-review}. We will henceforth focus on the class of malicious attacks known as Byzantine attacks~\cite{Lamport1982}. Robustifying federated learning against Byzantine attacks is of paramount importance for secure processing and learning.

To cope with Byzantine attacks in federated learning, several robust aggregation rules have been developed in recent years, mainly towards improving the distributed stochastic gradient descent (SGD) solver of the underlying optimization task. Through aggregating stochastic gradients with the geometric median \cite{Minsker2015GeometricMA,Chen2019DistributedSM}, median \cite{Xie2018GeneralizedBS}, trimmed mean \cite{Yin2018ByzantineRobustDL}, or iterative filtering \cite{Su2018Securing}, stochastic algorithms have been able to tolerate a small number of devices attacked by Byzantine adversaries.
Other aggregation rules include Krum \cite{Blanchard2017MachineLW}, that selects a stochastic gradient having the minimal cumulative squared distance from a given number of nearest stochastic gradients, and RSA \cite{Li2019RSABS} which aggregates models other than stochastic gradients through penalizing the differences between the local and global model parameters. Related works also include adversarial learning in distributed principal component analysis \cite{Feng2014DistributedRL}, escaping from saddle points in non-convex distributed learning under Byzantine attacks \cite{Yin2018Defending}, and leveraging redundant gradients to improve robustness \cite{Chen2018DRACO,Rajput2019DETOX}.

Although robust SGD iterates can ensure convergence to a neighborhood of the attack-free optimal solution, this neighborhood size can be large when Byzantine attacks are carefully crafted \cite{Xie2019FallOB}. Essentially, SGD suffers from the sizeable approximation error (noise) associated with stochastic gradients. This leads to the challenge of distinguishing malicious messages sent by Byzantine attackers from the noisy stochastic gradients sent by `honest' devices.

In the face of this challenge, we posed the following question: \textit{Is it possible to better distinguish the malicious messages from the stochastic gradients through reducing the stochastic gradient-induced noise?} Our answer will turn out to be in the affirmative. Intuitively, if the stochastic gradient noise is small, the malicious messages should be easy to identify; see also the illustrative example in Section \ref{subsec:noise}. This intuition suggests combining variance reduction techniques with robust aggregation rules to handle Byzantine attacks in federated learning.

Existing variance reduction techniques in stochastic optimization include mini-batch \cite{Goyal2017Accurate}, and abbreviated ones as SAG \cite{Schmidt2017MinimizingFS}, SVRG \cite{Johnson2013AcceleratingSG}, SAGA \cite{Defazio2014SAGAAF}, SDCA \cite{ShalevShwartz2013StochasticDC}, SARAH \cite{Nguyen2017Sarah}, Katyusha \cite{Zhu2017KatyushaTF}, to list a few. Among these, we are particularly interested in SAGA, which has been proven effective in finite-sum optimization. SAGA can also be implemented in a distributed manner \cite{Calauznes2017DistributedSM,De2016EfficientDS,Reddi2015OnVR}, and hence it fits well the federated learning applications, where each device deals with a finite number of data samples.

Our proposed novel Byzantine attack resilient distributed (Byrd-) SAGA combines SAGA's variance reduction with robust aggregation to deal with the malicious attacks in federated finite-sum optimization setups. Instead of the mean employed by distributed SAGA, the datacenter in Byrd-SAGA relies on the geometric median to aggregate the corrected stochastic gradients sent by distributed devices. Through reducing the stochastic gradient-induced noise, Byrd-SAGA turns out to outperform the Byzantine attack resilient distributed SGD. When less than half of the workers are Byzantine attackers, the robustness of geometric median to outliers enables Byrd-SAGA to achieve provably linear convergence to a neighborhood of the optimal solution, and the asymptotic learning error is solely determined by the number of Byzantine workers. Numerical tests demonstrate the robustness of Byrd-SAGA to various Byzantine attacks.


\section{Problem Statement}
\label{sec:statement}

We start this section by specifying the federated finite-sum optimization problem in the presence of Byzantine attacks. We then elaborate on the limitations of Byzantine attack resilient distributed SGD algorithms, which motivate our subsequent development of Byrd-SAGA.

\subsection{Federated finite-sum optimization in the presence of Byzantine attacks}

Consider a network with one master node (datacenter) and $W$ workers (devices), among which $B$ workers are Byzantine attackers with their identities unknown to the master node. Let $\mathcal{W}$ be the set of all workers, and  $\mathcal{B}$ that of Byzantine attackers with respective cardinalities  $|\mathcal{W}|=W$ and $|\mathcal{B}|=B$. The data samples are evenly distributed across the honest workers $w \notin \mathcal{B}$. Each honest worker has $J$ data samples, and $f_{w,j}(x)$ denotes the loss of the $j$-th data sample at the honest worker $w$ with respect to the model parameter $x \in \mathbb{R}^p$. We are interested in the finite-sum optimization problem
\begin{align}  \label{problem}
x^* = \arg\min_x ~ f(x) := \frac{1}{W-B} \sum_{w\notin B} f_{w}(x)
\end{align}
where
\begin{align}  \label{functions}
f_{w}(x) :=\avg{j}{J}f_{w, j}(x).
\end{align}
The main challenge of solving \eqref{problem} is that the Byzantine attackers can collude and send arbitrary malicious messages to the master node so as to bias the optimization process. We aspire to develop a robust distributed stochastic algorithm to address this issue. Intuitively, when a majority of workers are Byzantine attackers, it is difficult to obtain a reasonable approximate solution to \eqref{problem}. For this reason, we will assume $B < \frac{W}{2}$ throughout, and prove that the proposed Byzantine attack resilient algorithm is able to tolerate attacks from up to half of the workers.

\subsection{Sensitivity of distributed SGD to Byzantine attacks}

When all workers are honest, a popular solver of \eqref{problem} is SGD \cite{Bottou2010LargeScaleML}. At time slot (iteration) $k$, the master node broadcasts $x^k$ to workers. Upon receiving $x^k$, worker $w$ uniformly at random chooses a local data sample with index $i_w^k$ to obtain the stochastic gradient $f_{w,i_w^k}'(x^k)$ that then communicates back to the master node. Upon collecting stochastic gradients from all workers, the master node updates the model as
\begin{align}
\label{equation:SGDupdate}
x^{k+1} =  x^k-\gamma^k \cdot \avg{w}{W}f_{w,i_w^k}'(x^k)
\end{align}
where $\gamma^k$ is the non-negative step size. Note that the distributed SGD can be extended to its mini-batch version; whereby, each worker uniformly at random chooses a mini-batch of data samples per iteration, and communicates the averaged stochastic gradient back to the master node.


While the honest workers send true stochastic gradients to the master node, the Byzantine ones can send arbitrary malicious messages to the master node in order to perturb (bias) the optimization process. Let $m_w^k$ denote the message worker $w$ sends to the master node at slot $k$, given by
\begin{align}
m_w^k=
\begin{cases}
f_{w, i_w^k}'(x^k), \quad & w \notin\mathcal{B}, \\
*, \quad & w \in\mathcal{B}
\end{cases}
\label{definition:mk_SGD}
\end{align}
where $*$ denotes an arbitrary $p \times 1$ vector. Then, the distributed SGD update \eqref{equation:SGDupdate} becomes
\begin{align}   \label{equation:SGDupdate-true}
x^{k+1}  =  x^k - \gamma^k \cdot \avg{w}{W} m_w^k.
\end{align}

Even when only one Byzantine attacker is present, the distributed SGD may fail. Consider that a Byzantine attacker $w_b$ sends to the master node $m_{w_b}^k = - \sum_{w \neq w_b} m_w^k$, which yields $x^{k+1} = x^{k}$. In practice, Byzantine attackers can send more sophisticated messages to fool the master node, and thus bias the optimization process.

\subsection{Byzantine attack resilient distributed SGD}

Recent works often robustify the distributed SGD by incorporating robust aggregation rules when the master node receives messages from the workers. Here, we will adopt and analyze the geometric median, even though alternative robust aggregation rules are also viable \cite{Minsker2015GeometricMA,Chen2019DistributedSM}.

With $\{z, z \in \mathcal{Z}\}$ denoting a subset in a normed space, the geometric median of $\{z, z \in \mathcal{Z}\}$ is
\begin{align} \label{definition:geo}
    \underset{z \in \mathcal{Z}}{{\rm geomed}}\{z \} := \arg\min_{y} \sum_{z \in \mathcal{Z}}\|y-z\|.
\end{align}
Using \eqref{definition:geo}, the distributed SGD in \eqref{equation:SGDupdate-true} can be modified to its Byzantine attack resilient form as
\begin{align}   \label{equation:SGDupdate-robust}
x^{k+1}  =  x^k - \gamma^k \cdot \underset{w \in \mathcal{W}}{{\rm geomed}}\{m_w^k\}.
\end{align}
In essence, the geometric median chooses a reliable vector to represent the received messages $\{m_w^k\}$ through majority voting. When the number of Byzantine workers $B < \frac{W}{2}$, the geometric median approximates reasonably well the mean of $\{m_w^k, w \notin \mathcal{B}\}$. This property enables the Byzantine attack resilient distributed SGD to converge to a neighborhood of the optimal solution~\cite{Minsker2015GeometricMA,Chen2019DistributedSM}.

\begin{figure}
	\centering
	\includegraphics[width=0.45\textwidth]{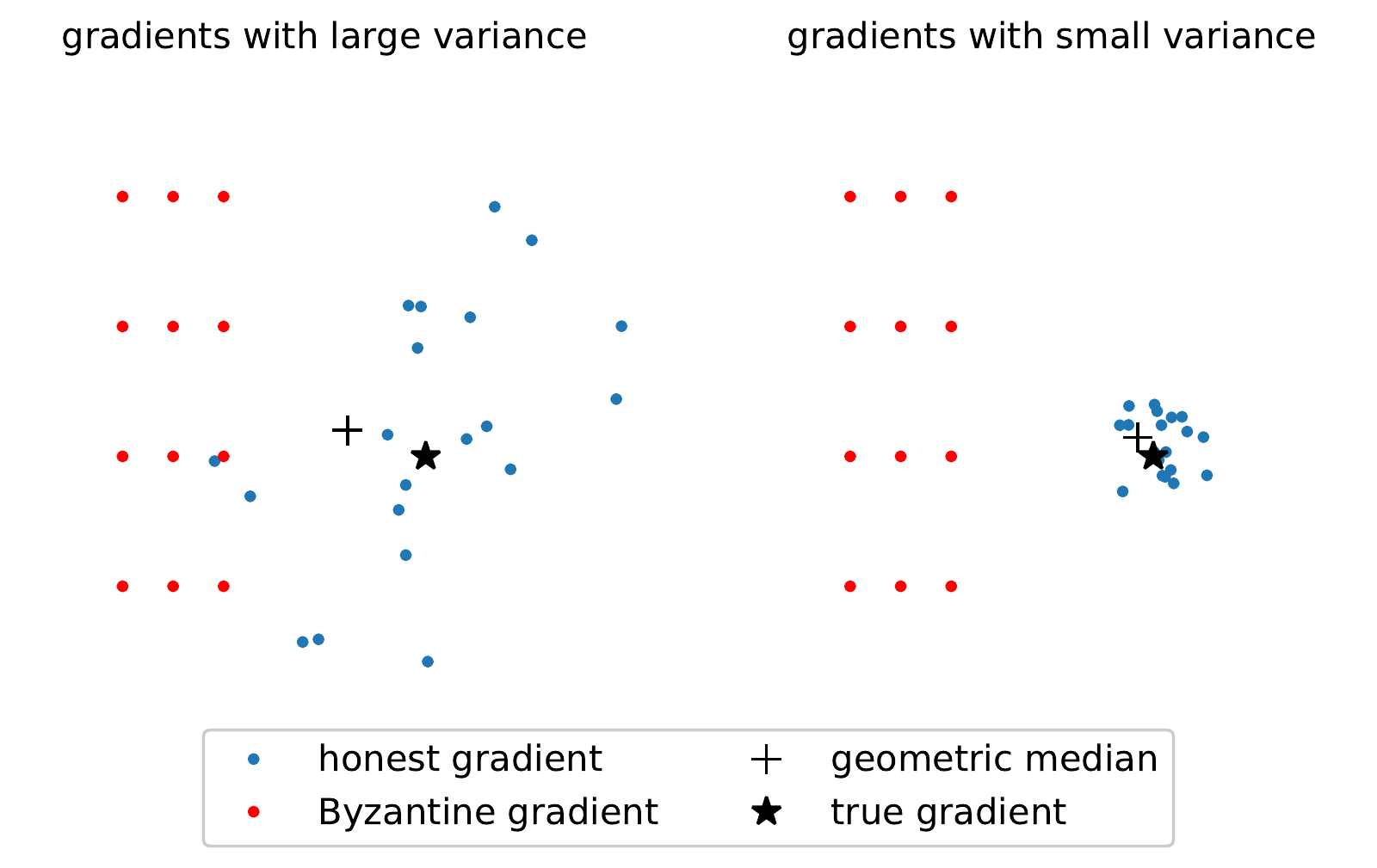}	
	\caption{Impact of stochastic gradient noise on geometric median-based robust aggregation. Blue dots denote stochastic gradients sent by the honest workers. Red dots denote malicious messages sent by the Byzantine workers. Plus signs denote the outputs of geometric median-based robust aggregation. Pentagrams denote the means of the stochastic gradients sent by the honest workers. Variance of the stochastic gradients from the honest workers is large in left and small in right.}
	\label{fig:howVarianceAffectGM}
\end{figure}

\subsection{Impact of stochastic gradient noise on robust aggregation}
\label{subsec:noise}

In distributed SGD, the stochastic gradients evaluated by honest workers are noisy because of the randomness in choosing data samples. Due to the stochastic gradient noise however, it is not always easy to distinguish the malicious messages from the stochastic gradients using just the robust aggregation rules, e.g. the geometric median. Several existing works have recognized this issue. With carefully crafted Byzantine attacks, outputs of several Byzantine attack resilient SGD algorithms can be far away from the optimal solution \cite{Xie2019FallOB}. In \cite{Chen2019DistributedSM} and \cite{Chen2018DRACO}, the workers are divided into several groups, with averages taken within groups and the geometric median obtained across groups. This approach leads to reduced variance and thus enhanced ability to distinguish malicious messages. In \cite{Blanchard2017MachineLW}, it is explicitly assumed that the ratio of the variance of stochastic gradients to the distance between iterate and optimal solution is upper-bounded.

Fig. \ref{fig:howVarianceAffectGM} shows the impact of stochastic gradient noise on geometric median-based robust aggregation. When the stochastic gradients sent by honest workers have small variance, the gap between the true mean and the aggregated value is also small; that is, the same Byzantine attacks are less effective. We will quantify this statement in our analysis of Section \ref{subsec:noise-analysis}.

Prompted by this observation, our key idea is to reduce the variance of stochastic gradients in order to enhance robustness to Byzantine attacks. In the Byzantine attack-free case, an effective approach to alleviating stochastic gradient noise in SGD is through variance reduction. By compensating for stochastic gradient noise, variance reduction techniques lead to faster convergence than SGD. For specificity, we will focus on SAGA, which reduces stochastic gradient noise for finite-sum optimization \cite{Defazio2014SAGAAF}, and we will show how SAGA can also aid robust aggregation against Byzantine attacks.



\section{Algorithm Development}
\label{sec:algorithm}

In this section, we first introduce distributed SAGA with mean aggregation. Then, we propose Byrd-SAGA, which replaces mean aggregation by geometric median-based robust aggregation.

\subsection{Distributed SAGA with mean aggregation}

In distributed SAGA, each worker maintains a table of stochastic gradients for all of its local data samples~\cite{Calauznes2017DistributedSM,De2016EfficientDS}. As in distributed SGD, the master node at slot $k$ sends $x^k$ to the workers, and every worker $w$ uniformly at random chooses a local data sample with index $i_w^k$ to find the stochastic gradient $f_{w,i_w^k}'(x^k)$. However, worker $w$ does not send back $f_{w,i_w^k}'(x^k)$ to the master node. Instead, it corrects $f_{w,i_w^k}'(x^k)$ by first subtracting the previously stored stochastic gradient of the $i_w^k$-th data sample, and then adding the average of the stored stochastic gradients across local data samples. Then, worker $w$ sends such a corrected stochastic gradient to the master node, and stores $f_{w,i_w^k}'(x^k)$ as the stochastic gradient of the $i_w^k$-th data sample in the table. After collecting the corrected stochastic gradients from all workers, the master node updates the model $x^{k+1}$.

To better describe distributed SAGA, let
\begin{align} \label{equation:phi}
\phi_{w,j}^{k+1}  =
\begin{cases}
\phi_{w,j}^k, \quad & j \neq i_w^k \\
x^k, \quad & j = i_w^k
\end{cases}
\end{align}
where $\phi_{w,j}^{k+1}$ is the iterate at which the most recent $f_{w,j}'$ is evaluated when slot $k$ ends. Then, $f_{w,j}'(\phi_{w,j}^k)$ refers to the previously stored stochastic gradient of the $j$-th data sample prior to slot $k$ on worker $w$, and
\begin{align}
g_w^k := f_{w, i_w^k}'(x^k)-f_{w, i_w^k}'(\phi_{w,i_w^k}^k) +\avg{j}{J} f_{w, j}'(\phi_{w,j}^k) \nonumber
\end{align}
is the corrected stochastic gradient of worker $i$ at slot $k$. The model update of SAGA is hence
\begin{align}   \label{equation:SAGAupdate}
& x^{k+1}  =  x^k - \gamma \cdot \avg{w}{W} g_w^k
\end{align}
where $\gamma > 0$ is the constant step size.

\subsection{Distributed SAGA with geometric median aggregation}

Here, it is useful to recall that Byzantine workers may send to the master node malicious messages, other than the corrected stochastic gradient. To account for this, the message sent from worker $w$ to the master node at slot $k$ is expressed as
\begin{align}
m_w^k&=
\begin{cases}
g_w^k, & w \notin\mathcal{B}, \\
*,     & w \in\mathcal{B}
\end{cases}
\label{definition:mk}
\end{align}
where $*$ denotes an arbitrary $p \times 1$ vector. Similar to distributed SGD, distributed SAGA is also sensitive to Byzantine attacks. Our robust aggregation rule here is the geometric median. This leads to the proposed Byzantine attack resilient distributed (Byrd) form of SAGA in \eqref{equation:SAGAupdate}, that is given by
\begin{align}   \label{equation:SAGAupdate-robust}
x^{k+1}  =  x^k - \gamma \cdot \underset{w \in \mathcal{W}}{{\rm geomed}}\{m_w^k\}.
\end{align}


\begin{figure}
	\centering
	\includegraphics[width=0.45\textwidth]{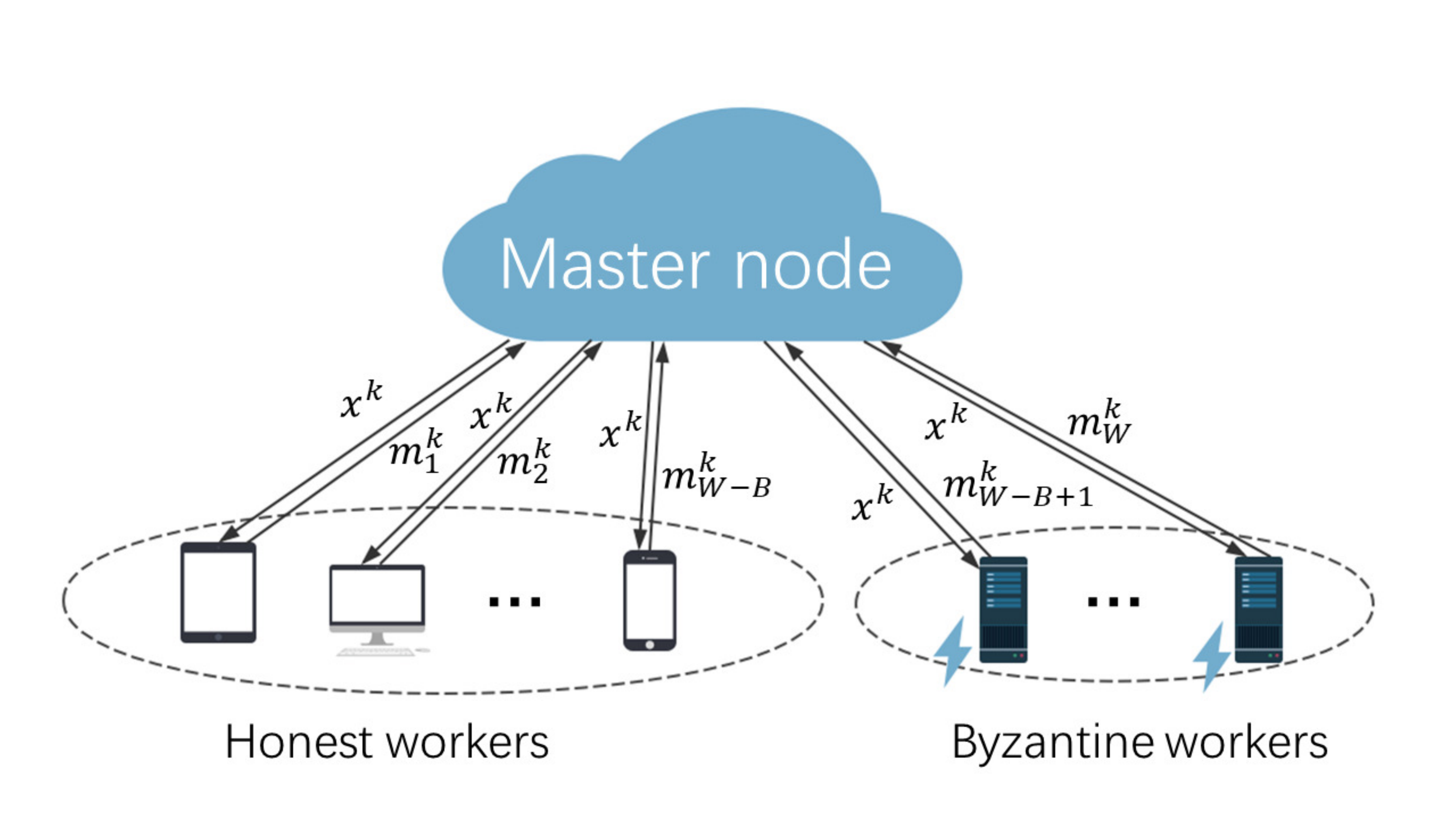}
	\caption{Illustration of Byzantine attack resilient distributed SAGA. For the ease of illustration, the honest workers are from $1$ to $W-B$ while the Byzantine attackers are from $W-B+1$ to $W$. But in practice, the identities of Byzantine attackers are unknown to the master node.}
	\label{fig:framework}
\end{figure}

%
%

The proposed Byzantine attack resilient distributed SAGA, abbreviated as Byrd-SAGA, is listed step-by-step under Algorithm \ref{algorithm:Byrd-SAGA}, and illustrated in Fig. \ref{fig:framework}. There are various implementations of the distributed SAGA. For example, \cite{De2016EfficientDS} proposed to store the tables of stochastic gradients in the master node. The workers only need to upload the stochastic gradients and their indexes, while the master node performs the aggregation. This setup is also vulnerable to Byzantine attacks, since the Byzantine attackers may upload incorrect stochastic gradients. The proposed robust aggregation rule can also be applied therein.


\begin{algorithm}[!tb]
	\caption{Byzantine Attack Resilient Distributed SAGA}
	\label{algorithm:Byrd-SAGA}
	\begin{algorithmic}[0]
		\REQUIRE step size $\gamma$; number of workers $W$; number of data samples $J$ on every honest worker $w$
		\STATE Master node and honest workers initialize $x^0$
		\FORALL {honest worker $w$}
		\FOR {$j \in \{1,\dots, J\}$}
		\STATE Initializes gradient storage $f_{w, j}'(\phi_{w,j}) =  f_{w, j}'(x^0)$
		\ENDFOR
		\STATE Initializes average gradient $\bar{g}_w^1=\frac{1}{J}\sum_{j=1}^{J}f_{w, j}'(x^0)$
		\STATE Sends $\bar{g}_w^1$ to master node
		\ENDFOR
		\STATE Master node updates $x^1  =  x^0 - \gamma \cdot {\rm geomed}_{w \in \mathcal{W}} \{ \bar{g}_w^1 \}$
		\FORALL {$k = 1, 2, \cdots $}
		\STATE Master node broadcasts $x^k$ to all workers
		\FORALL {honest worker node $w$}
		\STATE Samples $i_w^k$ from $\{1, \cdots, J\}$ uniformly at random
		\STATE Updates $m_w^k = f_{w, i_w^k}'(x^k)-f_{w, i_w^k}'(\phi_{w,i_w^k})+\bar{g}_w^k$
		\STATE Sends $m_w^k$ to master node
		\STATE Updates $\bar{g}_w^{k+1} = \bar{g}_w^k+\frac{1}{J}(f_{w, i_w^k}'(x^k)-f_{w, i_w^k}'(\phi_{w,i_w^k}))$
		\STATE Stores gradient $f_{w, i_w^k}'(\phi_{w,i_w^k}) =  f_{w, i_w^k}'(x^k)$
		\ENDFOR \\
		\STATE Master node updates $x^{k+1}  =  x^k - \gamma \cdot {\rm geomed}_{w \in \mathcal{W}} \{ m_w^k \}$
		\ENDFOR
	\end{algorithmic}
\end{algorithm}

Robust aggregations other than the geometric median are available, including the median \cite{Xie2018GeneralizedBS}, Krum \cite{Blanchard2017MachineLW}, marginal trimmed mean \cite{Yin2018ByzantineRobustDL}, and iterative filtering \cite{Su2018Securing}. In the median for instance, the aggregation outputs the element-wise median of $\{m_w^k\}$; while in the Krum, the aggregation outputs
\begin{align*}
\underset{w\in\mathcal{W}}{{\rm Krum}} \{m_w^k\} = m_{w^*}, \
w^*=\arg\min_{w\in\mathcal{W}}\sum_{w\to w'}\|m_w^k-m_{w'}^k\|^2
\end{align*}
where $w\to w'$ $(w \neq w')$ selects the indexes $w'$ of the $W-B-2$ nearest neighbors of $m_w^k$ in $\{m_{w'}^k\}$. Note that Krum needs to know $B$, the number of Byzantine attackers, in advance. In addition, other variance reduction techniques, such as mini-batch \cite{Goyal2017Accurate}, SAG \cite{Schmidt2017MinimizingFS}, SVRG \cite{Johnson2013AcceleratingSG}, SAGA \cite{Defazio2014SAGAAF}, SDCA \cite{ShalevShwartz2013StochasticDC}, SARAH \cite{Nguyen2017Sarah} and Katyusha \cite{Zhu2017KatyushaTF}, are also available to alleviate the gradient noise. Here we opted for the combination of geometric median and SAGA. Extending the current work to other robust aggregation rules and variance reduction techniques, is in our future research agenda.


\begin{Remark}
Computing the geometric median involves solving an optimization problem in the form of \eqref{definition:geo}. Since it is costly to obtain the exact geometric median, one is typically
satisfied with an $\epsilon$-approximate value \cite{Weiszfeld2009}. We say that
$z^*_\epsilon$ is an $\epsilon$-approximate geometric median of $\{z, z \in \mathcal{Z}\}$ if
	\begin{align}
	\sum_{z \in \mathcal{Z}}\|z^*_\epsilon - z\|
	\le \inf_{y} \sum_{z \in \mathcal{Z}}\|y-z\|+\epsilon.
	\end{align}
	We shall show that the $\epsilon$-approximation only slightly affects the convergence of Byrd-SAGA.
\end{Remark}

\section{Theoretical Analysis}
\label{sec:analysis}

In this section, we theoretically justify the intuitive idea that reducing stochastic gradient noise helps identify malicious messages in robust aggregation, specifically to the geometric median in this paper. We prove that our Byrd-SAGA converges to a neighborhood of the optimal solution at a linear rate under Byzantine attacks, and the asymptotic learning error is determined by the number of Byzantine attackers. Due to the page limit, proofs are delegated to the full version of this paper\footnote{\url{https://github.com/MrFive5555/Byrd-SAGA/blob/master/Full.pdf}}.


\subsection{Importance of reducing stochastic gradient noise}
\label{subsec:noise-analysis}

Here, we quantify the role of stochastic gradient noise on the geometric median aggregation. Towards this objective, consider the set of messages $\mathcal{Z}$ sent by all workers in $\mathcal{W}$, and the set $\mathcal{Z}'$ of malicious messages sent by the Byzantine attackers in $\mathcal{B}$. Further, let $\bar{z}$ denote the true gradient given by the ensemble average of stochastic gradients. Using these definitions, the ensuing lemma bounds the mean-square error of the geometric median relative to the true gradient.

\begin{Lemma} \label{lemma:geometric_error} (Concentration property)
	Let $\{z, z \in \mathcal{Z}\}$ be a subset of random vectors distributed in a normed vector space. If $\mathcal{Z}'\subseteq \mathcal{Z}$ and $|\mathcal{Z}'|< \frac{|\mathcal{Z}|}{2}$, then it holds that
	\begin{align} \label{inequality:geovoting-error-exp}
	    & E \|\underset{z \in \mathcal{Z}}{{\rm geomed}}\{z\}- \bar{z}\|^2  \\
	\le & C_\alpha^2 \frac {\sum_{z\notin \mathcal{Z}'}{ E\|z-E z\|^2 }} {|\mathcal{Z}|-|\mathcal{Z}'|}
	+ C_\alpha^2 \frac {\sum_{z\notin \mathcal{Z}'}{ \|E z- \bar{z}\|^2}} {|\mathcal{Z}|-|\mathcal{Z}'|} \nonumber
	\end{align}
	where
	\begin{align}
	\bar{z}:=\frac {\sum_{z \notin \mathcal{Z}'} E z} {|\mathcal{Z}|-|\mathcal{Z}'|} \nonumber
	\end{align}
while $C_\alpha :=\frac{2-2\alpha}{1-2\alpha}$, and $\alpha :=\frac{|\mathcal{Z}'|}{|\mathcal{Z}|}$.
\end{Lemma}

The left-hand side of \eqref{inequality:geovoting-error-exp} is the mean-square error of the geometric median relative to the true gradient, while the right-hand side is the sum of two terms. The first is determined by the variances of the local stochastic gradients sent by the honest workers (inner variation), while the second term is determined by the variations of the local gradients at the honest workers with respect to the true gradient (outer variation). In the Byzantine attack resilient SGD, the upper bound can be large due to the large stochastic gradient noise of SGD. Through reducing the stochastic gradient noise in terms of either inner variation or outer variation, we are able to attain improved accuracy under malicious attacks.

\subsection{Convergence of Byrd-SAGA and comparison with Byzantine attack resilient SGD}
\label{subsec:convergence}

Here, we establish convergence of Byrd-SAGA, and theoretically justify that, through reducing the impact of inner variation, Byrd-SAGA enjoys superior robustness to Byzantine attacks. We begin with several needed assumptions on the functions $\{f_{w,j}\}$.

\begin{Assumption} (Strong convexity and Lipschitz continiuty of gradients)   \label{assumption:muL}
	The function $f$ is $\mu$-strongly convex and has $L$-Lipschitz continuous gradients, which amounts to requiring that for any $x, y \in \mathbb{R}^p$, it holds that
	\begin{align}
	f(x)\ge f(y)+\langle f'(y),x-y\rangle +\frac \mu 2\|x-y\|^2 \label{eq:mu}
	\end{align}
	and
	\begin{align}
	\|f'(x)-f'(y)\|\le L\|x-y\|. \label{eq:L}
	\end{align}
\end{Assumption}

\begin{Assumption} (Bounded outer variation)  \label{assumption:outterVariance}
	For any $x \in \mathbb{R}^p$, variation of the aggregated gradients at the honest workers with respect to the overall gradient is upper-bounded by
	\begin{align}
	E_{w\notin\mathcal{B}}\|f_w'(x)-f'(x)\|^2\le\delta^2. \label{eq:outterVariance}
	\end{align}
\end{Assumption}

\begin{Assumption} (Bounded inner variation)  \label{assumption:innerVariance}
	For every honest worker $w$ and any $x \in \mathbb{R}^p$, the variation of its stochastic gradients with respect to its aggregated gradient is upper-bounded by
	\begin{align}
	E_{i_w^k}\|f_{w,i_w^k}'(x)-f_{w}'(x)\|^2\le\sigma^2,\quad \forall w \notin B. \label{eq:innerVariance}
	\end{align}
\end{Assumption}

Assumption \ref{assumption:muL} is standard in convex analysis. Assumptions \ref{assumption:outterVariance} and \ref{assumption:innerVariance} bound the variation of gradients and the variation of stochastic gradients within the honest workers, respectively \cite{Tang2018D2DT}. For instance, most of the existing Byzantine attack resilient SGD algorihtms assume that the stochastic gradients at the honest workers are independently and identically distributed (i.i.d.) with finite variance, such that the outer variation $\delta^2$ in Assumption \ref{assumption:outterVariance} is proportional to $1/J$ and the inner variation $\sigma^2$ in Assumption \ref{assumption:innerVariance} is finite. In the analysis of Byzantine attack resilient SGD, \emph{both} outer and inner variations must be bounded. Interestingly, inner variation will turn out not to impact Byrd-SAGA, and Assumption \ref{assumption:innerVariance} will no longer be necessary in its analysis.

To simplify notation, we will henceforth use $E$ to represent the expectation with respect to all random variables $i_w^k$.

The presence of geometric median makes Byrd-SAGA analysis challenging. Specifically, for every honest worker $w \notin \mathcal{B}$, $m_i^k$ is an unbiased estimate of $f_w'(x^k)$, meaning
\begin{align}
E m_i^k =f_w'(x^k).
\label{equation:unbiasUpdadte-00}
\end{align}
Averaging \eqref{equation:unbiasUpdadte-00} over all honest workers $w \notin \mathcal{B}$, we have
\begin{align}
\avgHonest E m_i^k = \avgHonest f_w'(x^k) =f'(x^k).
\label{equation:unbiasUpdadte}
\end{align}
From \eqref{equation:unbiasUpdadte}, we observe that the mean of $\{m_i^k\}$ over all the honest workers $w \notin \mathcal{B}$ is an unbiased estimate of $f'(x^k)$. Nevertheless, the geometric median of $\{m_i^k\}$, even only over all the honest workers $w \notin \mathcal{B}$ and calculated accurately, is a biased estimate of $f'(x^k)$. This is the main challenge in adapting the proof of SAGA to that of Byrd-SAGA.

The following theorem asserts that Byrd-SAGA converges to a neighborhood of the optimal solution $x^*$ at a linear rate, with the asymptotic learning error determined by the number of Byzantine attackers.

\begin{Theorem}
	Under Assumptions \ref{assumption:muL} and \ref{assumption:outterVariance}, if the number of Byzantine attackers satisfies $B < \frac{W}{2}$ and the step size satisfies
	$$\gamma \le \frac{\mu}{4\sqrt{5}J^{2}C_\alpha L^2}$$
	then for Byrd-SAGA with $\epsilon$-approximate geometric median aggregation, it holds that
	\begin{align}
	E\|x^k-x^*\|^2 \le (1- \frac{\gamma\mu}{2})^{k} \Delta_1
	+\Delta_2
	\label{inequality:xk}
	\end{align}
	where
	\begin{align} \label{definition:Delta1}
	\Delta_1 := \|x^0-x^*\|^2 - \Delta_2
	\end{align}
	\begin{align} \label{definition:Delta2}
	\hspace{-2em} \Delta_2 :=
	\frac{10}{\mu^2}
	\left(
	C_\alpha^2\delta^2
	+\frac{\epsilon^2}{(W-2B)^2}
	\right).
	\end{align}
	\label{theorem:convergence}
\end{Theorem}

In \eqref{inequality:xk}, the constant of convergence rate is given by
$$1-\frac {\gamma \mu} {2} \geq 1 - \frac {1} {4\sqrt{5}J^{2}C_\alpha \frac{L^2}{\mu^2}}$$
which is close to $1$ when $J$ (the number of data samples at each worker) and $\frac{L}{\mu}$ (the condition number of functions) are large. Observe that $C_\alpha$ is monotonically increasing when the portion of Byzantine attackers $\alpha$ increases. Therefore, \eqref{inequality:xk} shows that Byrd-SAGA converges slower as the number of Byzantine attackers grows. Correspondingly, the theoretical upper bound of step size $\gamma$ is small when $J$ and $C_\alpha$ are large. The asymptotic learning error $\Delta_2$ in \eqref{definition:Delta2} is also monotonically increasing when $C_\alpha$ (and hence the number of Byzantine attackers) increases.

To demonstrate the superior robustness of Byrd-SAGA, we also establish the convergence of Byzantine attack resilient SGD with constant step size as a benchmark. As in Theorem \ref{theorem:convergence}, the convergence of Byzantine attack resilient SGD is in the mean-square error sense. This is different from \cite{Chen2019DistributedSM}, where convergence is asserted in the high probability sense.

\begin{Theorem}
	Under Assumptions \ref{assumption:muL}, \ref{assumption:outterVariance} and \ref{assumption:innerVariance}, if the number of Byzantine attackers is $B < \frac{W}{2}$ and the step size satisfies $$\gamma < \frac{\mu}{2L^2}$$ then for Byzantine attack resilient SGD with $\epsilon$-approximate geometric median aggregation, it holds that
	\begin{align}
	E\|x^k-x^*\|^2 \le (1- \gamma\mu)^{k} \Delta_1'
	+\Delta_2'
	\label{inequality:xkSGD}
	\end{align}
	where
	\begin{align}
	\label{definition:Delta1SGD}
	\Delta_1' &:= \|x^0-x^*\|^2 -\Delta_2'
	\end{align}
	\begin{align}
	\label{definition:Delta2SGD}
	\Delta_2' &:= \frac{4}{\mu^2}\left(
	C_\alpha^2\sigma^2
	+C_\alpha^2\delta^2
	+\frac{\epsilon^2}{(W-2B)^2}
	\right).
	\end{align}.
	\label{theorem:convergenceSGD}
\end{Theorem}

Let us ignore the approximation error in computing geometric median by setting $\epsilon=0$, and compare the two asymptotic learning errors $\Delta_2$ and $\Delta_2'$. Therefore, we deduce that
\begin{align}
\Delta_2 = O\left(\frac{C_\alpha^2}{\mu^2} \delta^2\right) \quad \text{and} \quad \Delta_2' = O\left(\frac{C_\alpha^2}{\mu^2} (\sigma^2 + \delta^2)\right). \nonumber
\end{align}
Observe that $\Delta_2'$, the asymptotic learning error of Byzantine attack resilient SGD, is proportional to the sum of inner and outer variations. With all honest workers having the same data sample, we have $\sigma^2=\delta^2=0$. In this case, the asymptotic learning error $\Delta_2'$ vanishes because the geometric median aggregation takes effect and attains the true gradient. However, when each honest worker has the same set of distinct data samples, the inner variation $\sigma^2$ is no longer zero and the asymptotic learning error $\Delta_2'$ can be large. In contrast, Byrd-SAGA effectively reduces the impact of inner variation, and is able to achieve smaller learning error.

\begin{figure*}
	\centering
	\includegraphics[scale=0.45]{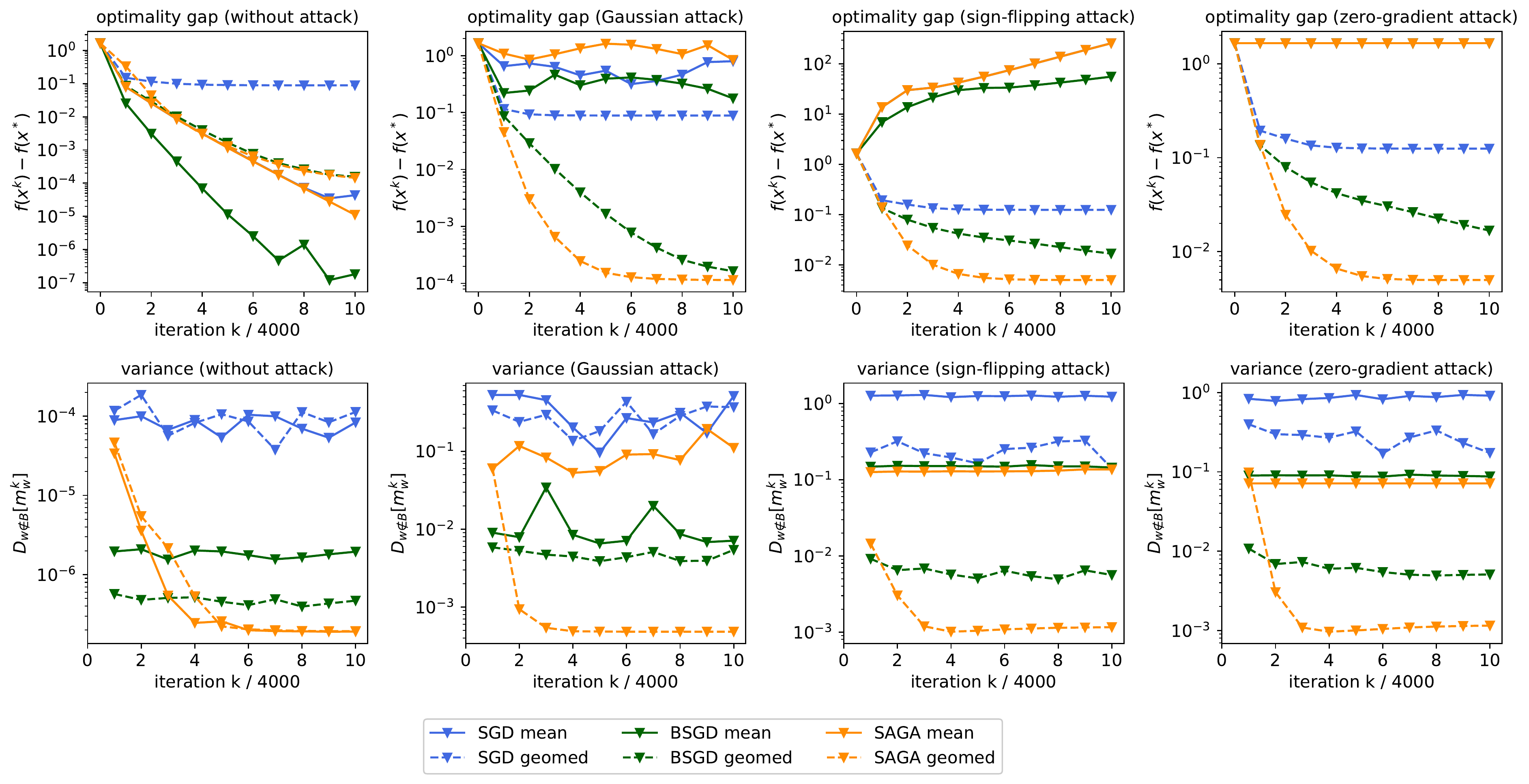}
	\caption{Performance of the distributed SGD, mini-batch (B)SGD and SAGA, with mean and geometric median (geomed) aggregation rules on IJCNN1 dataset. The step sizes are 0.02, 0.01 and 0.02, respectively. SAGA geomed stands for the proposed Byrd-SAGA. From top to bottom: optimality gap and variance of honest messages. From left to right: without attack, Gaussian attack, sign-flipping attack, and zero-gradient attack.}
	\label{fig:attack_SGD_ijcnn}
\end{figure*}
\begin{figure*}
	\centering
	\includegraphics[scale=0.45]{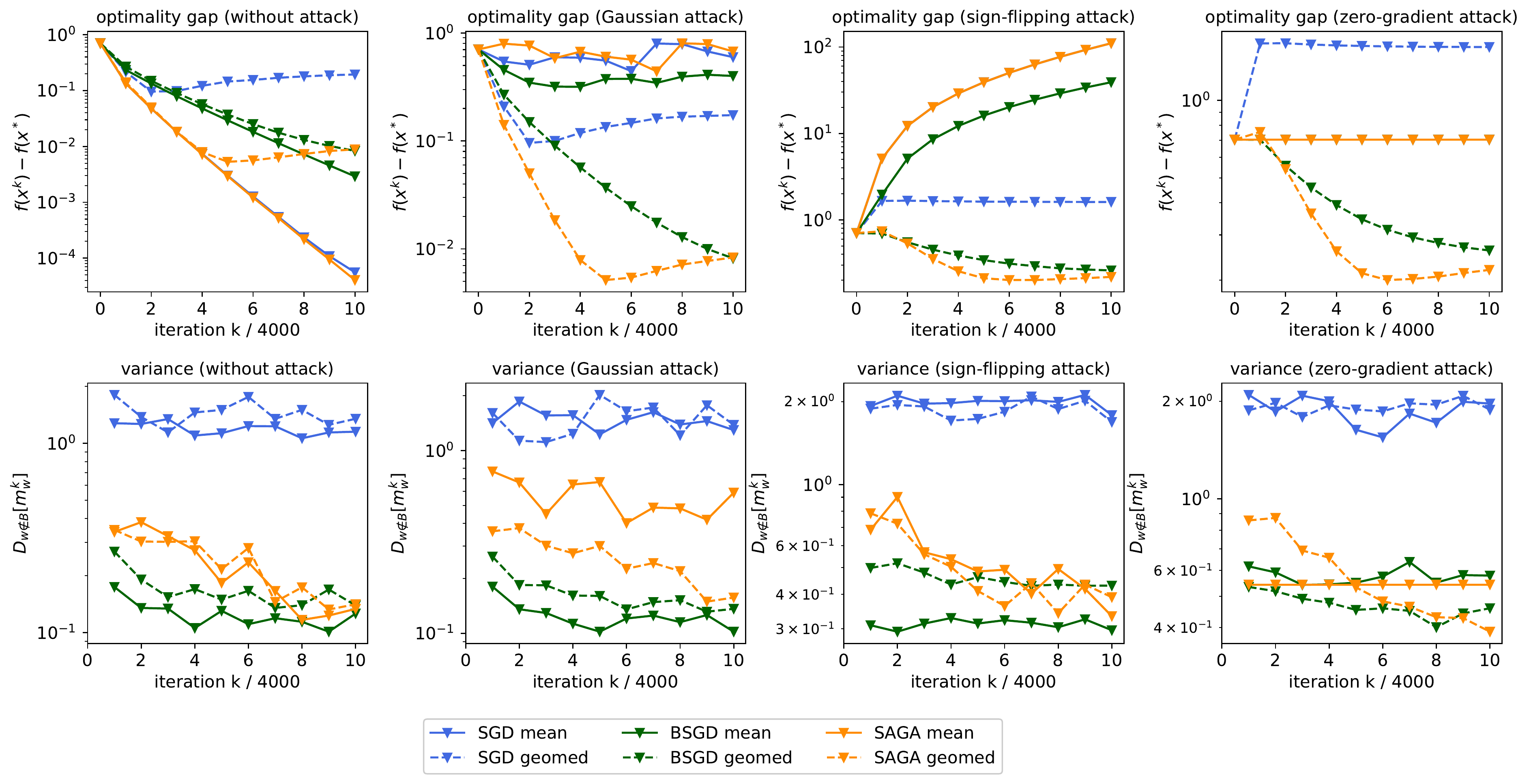}
	\caption{Performance of the distributed SGD, mini-batch (B)SGD and SAGA, with mean and geometric median (geomed) aggregation rules on COVTYPE dataset. The step sizes are 0.01, 0.005 and 0.01, respectively. SAGA geomed stands for the proposed Byrd-SAGA. From top to bottom: optimality gap and variance of honest messages. From left to right: without attack, Gaussian attack, sign-flipping attack, and zero-gradient attack.}
	\label{fig:attack_SGD_covtype}
\end{figure*}

\section{Numerical Experiments}
\label{sec:experiments}

Here we present numerical experiments on convex and nonconvex learning problems\footnote{The codes are available at \url{https://github.com/MrFive5555/Byrd-SAGA}}. For each problem, we evenly distribute the dataset into $W-B = 50$ honest workers unless indicated otherwise. To account for malicious attacks, we additionally launch $B=20$ Byzantine workers. We test the performance of the proposed Byrd-SAGA under three typical Byzantine attacks: Gaussian, sign-flipping and zero-gradient attacks \cite{Li2019RSABS,Lin2019}. For a Gaussian attack, a Byzantine attacker $w \in \mathcal{B}$ draws its $m_w^k$ from a Gaussian distribution with mean $\frac{1}{W-B} \sum_{w' \notin \mathcal{B}} m_{w'}^k$ and variance $30$. For a sign-flipping attack, a Byzantine attacker $w \in \mathcal{B}$ sets its message as $m_w^k = u \cdot \frac{1}{W-B} \sum_{w' \notin \mathcal{B}} m_{w'}^k$, where the magnitude $u=-3$ is used in the numerical experiments. And for a zero-gradient attack, a Byzantine attacker $w \in \mathcal{B}$ sends $m_w^k=- \frac{1}{B} \sum_{w' \notin \mathcal{B}} m_{w'}^k$ so that the messages at the master sum up to zero. We use the algorithm  in \cite{Weiszfeld2009} to obtain the $\epsilon$-approximate geometric median with $\epsilon=1\times 10^{-5}$.

\subsection{$\ell_2$-regularized logistic regression}

Consider the $\ell_2$-regularized logistic regression cost, where each summand $f_{w, j}(x)$ is given by
$$f_{w,i}(x)=\ln \left(1+\exp \left( -b_{w,i}\langle a_{w,i},x\rangle \right) \right)+\frac{\rho}{2}||x||^2$$
with $a_{w,j}\in \mathbb{R}^p$ being the feature vector, $b_{w,j}\in \{-1, 1\}$ the label, and $\rho = 0.01$ a constant. We use the IJCNN1 and COVTYPE datasets\footnote{\url{https://www.csie.ntu.edu.tw/~cjlin/libsvmtools/datasets}}. IJCNN1 contains 49,990 training data samples of $p=22$ dimensions. COVTYPE contains 581,012 training data samples of $p=54$ dimensions.

We first compare SGD, mini-batch (B)SGD with batch size $50$ and SAGA, using mean and geometric median aggregation rules. Compared to SGD, BSGD enjoys smaller stochastic gradient noise, but incurs higher computational cost. In comparison, SAGA also reduces stochastic gradient noise, but its computational cost is in the same order as that of SGD. For each algorithm, we adopt a constant step size, which is tuned to achieve the best optimality gap $f(x^k) - f(x^*)$ in the Byzantine-free scenario. The performance of these algorithms on the IJCNN1 and COVTYPE datasets is depicted in Fig. \ref{fig:attack_SGD_ijcnn} and Fig. \ref{fig:attack_SGD_covtype}, respectively. With Byzantine attacks, all three algorithms using mean aggregation fail. Among the three using geometric median aggregation, Byrd-SAGA markedly outperforms the other two, while BSGD is better than SGD. This demonstrates the importance of variance reduction to handling Byzantine attacks. Regarding the variance of honest messages in particular, Byrd-SAGA, Byzantine attack resilient BSGD and Byzantine attack resilient SGD are in the order of $10^{-3}$, $10^{-2}$ and $10^{-1}$, respectively, for the IJCNN1 dataset. For the COVTYPE dataset, Byrd-SAGA and Byzantine attack resilient BSGD have the same order of variance with respect to honest messages. In this case, Byrd-SAGA achieves similar optimality gap as Byzantine attack resilient BSGD, but converges faster because it is able to use a larger step size.

Theorem \ref{theorem:convergence} establishes that when the outer variation $\delta^2 = 0$, the asymptotic learning error of Byrd-SAGA is zero, no matter how large the inner variation $\sigma^2$ is. In contrast, according to Theorem \ref{theorem:convergenceSGD}, the asymptotic learning error of Byzantine attack resilient SGD is still proportional to the inner variation $\sigma^2$. To validate these theoretical results, we conducted a second set of numerical experiments, where every honest worker has the whole IJCNN1 dataset. Therefore, $\delta^2 = 0$ and $\sigma^2$ remains the same as that in the first set of experiments. We compare SGD, BSGD with batch size $50$ and SAGA, all using the geometric median aggregation rule. The results depicted in Fig. \ref{fig:zeroOuterVariation} corroborate the theoretical findings -- the asymptotic learning error of Byrd-SAGA vanishes, while those of Byzantine attack resilient SGD and BSGD are the same as those shown in Fig. \ref{fig:attack_SGD_ijcnn}.

\begin{figure*}
	\centering
	\includegraphics[scale=0.45]{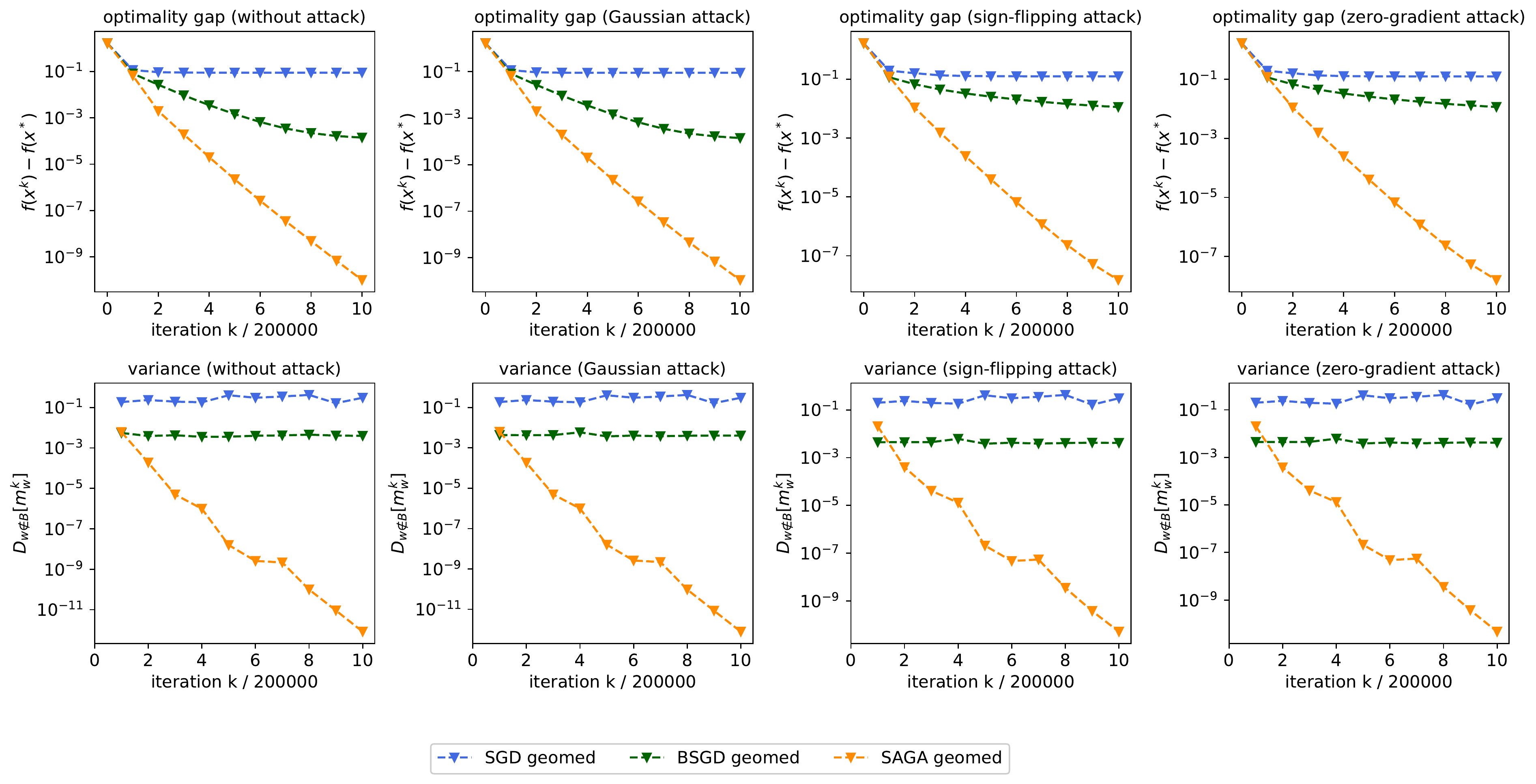}
	\caption{Performance of the distributed SGD, mini-batch (B)SGD and SAGA, with geometric median (geomed) aggregation rule. Every honest worker has the whole IJCNN1 dataset. The step sizes are 0.0004, 0.0002 and 0.0004, respectively. SAGA geomed stands for the proposed Byrd-SAGA. From top to bottom: optimality gap and variance of honest messages. From left to right: without attack, Gaussian attack, sign-flipping attack, and zero-gradient attack.}
	\label{fig:zeroOuterVariation}
\end{figure*}

In the third set of numerical experiments, we compare the use of different aggregation rules in distributed SAGA: mean, geometric median, median, and Krum. As shown in Fig. \ref{fig:attack}, distributed SAGA using mean aggregation is the best in terms of the optimality gap $f(x^k) - f(x^*)$ when there are no Byzantine attacks. However, it fails under all kinds of attacks. With Gaussian attacks, Byrd-SAGA using geometric median achieves the best performance. With sign-flipping and zero-gradient attacks, Byrd-SAGA using Krum is the best, while that using geometric median also performs well. Note that Krum has to know the exact number of Byzantine attackers in advance, while geometric median and median do not need this prior knowledge.

\begin{figure*}
	\centering
	\includegraphics[scale=0.4]{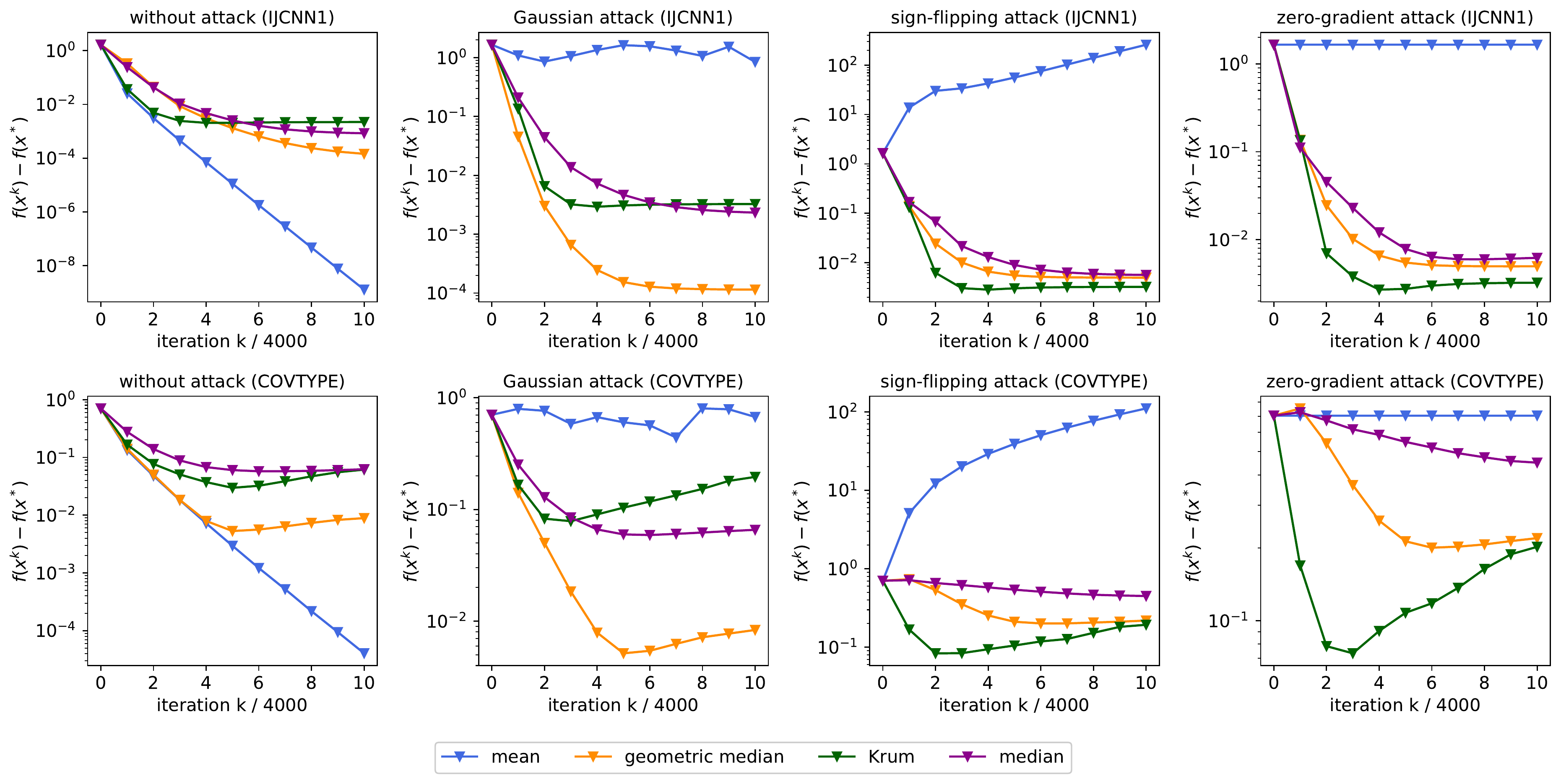}
	\caption{Optimality gaps of distributed SAGA with different aggregation rules: mean, geometric median, median and Krum. The step sizes are 0.02 and 0.01 for the IJCNN1 and COVTYPE datasets, respectively. Curves of geometric median correspond to the proposed Byrd-SAGA. From top to bottom: on IJCNN1 dataset and on COVTYPE dataset. From left to right: without attacks, with Gaussian attacks, with sign-flipping attacks, and with zero-gradient attacks.}
	\label{fig:attack}
\end{figure*}

\subsection{Neural network training}

Here we test training a neural network with one hidden layer of $50$ neurons and ``tanh'' activation function, for multi-class classification on the MNIST dataset\footnote{http://yann.lecun.com/exdb/mnist} comprising $60,000$ data samples, each with dimension $p=784$. We compare SGD with step size $0.1$, BSGD with step size $0.5$ and batch size $50$, and SAGA with step size $0.1$. We run the algorithms for $15,000$ iterations, and report the final accuracy in Table 1. With mean aggregation, all algorithms yield low accuracy in the presence of Byzantine attacks. With the help of geometric median aggregation, BSGD and SAGA are both robust and outperform SGD. Note that Byrd-SAGA exhibits a much lower per-iteration computational cost relative to Byzantine attack resilient BSGD.

\begin{table}
	\centering
	\label{Table:MLP}
	\caption{Accuracy of SGD, mini-batch (B)SGD and SAGA, with mean and geometric median (geomed) aggregation rules. SAGA geomed stands for the proposed Byrd-SAGA.}
	\begin{tabular}{cccc}
		\hline
		attack & algorithm & mean acc (\%) & geomed acc (\%) \\
		\hline
		without & SGD & 97.0 & 92.3\\
		& BSGD & 98.6 & 98.0\\
		& SAGA & 96.5 & 96.3\\
		\hline
		Gaussian & SGD & 36.3 & 92.5\\
		& BSGD & 36.3 & 98.0\\
		& SAGA & 14.5 & 96.4\\
		\hline
		sign-flipping & SGD & 0.11 & 0.03\\
		& BSGD & 0.16 & 90.3\\
		& SAGA & 0.12 & 86.4\\
		\hline
		zero-gradient & SGD & 9.94 & 26.2\\
		& BSGD & 9.89 & 81.5\\
		& SAGA & 9.88 & 92.4\\
		\hline
	\end{tabular}
\end{table}

\section{Conclusions}
\label{sec:conclusions}

The present paper developed a novel Byzantine attack resilient distributed (Byrd-) SAGA approach to federated finite-sum optimization in the presence of Byzantine attacks. On par with SAGA, Byrd-SAGA corrects stochastic gradients through variance reduction. Per iteration, distributed workers obtain their corrected stochastic gradients before uploading to the master node. Different from SAGA though, the master node in Byrd-SAGA aggregates the received messages using the geometric median rather than the mean. This robust aggregation markedly enhances robustness of Byrd-SAGA in the presence of Byzantine attacks. It was established that Byrd-SAGA converges linearly to a neighborhood of the optimal solution, with the asymptotic learning error determined solely by the number of Byzantine workers.

As confirmed by numerical tests, combinations with other robust aggregation rules also exhibit satisfactory robustness. Our future research agenda includes their analysis, as well as the development and analysis of Byzantine attack resilient algorithms over fully decentralized networks \cite{Bianchi2016,Bajwa2019}.

\clearpage
\onecolumn
\appendices

\section{Proof of Lemma \ref{lemma:geometric_error}}
The proof of Lemma \ref{lemma:geometric_error} relies on the following lemma.
\begin{Lemma}
	Let $\{z : z \in \mathcal{Z}\}$ be a subset of random vectors distributed in a normed vector space. If $\mathcal{Z}'\subseteq \mathcal{Z}$ and $|\mathcal{Z}'|< \frac{|\mathcal{Z}|}{2}$, then it holds that
	\begin{align}
	E\|\underset{z \in \mathcal{Z}}{{\rm geomed}}\{z\}\|^2 \le C_\alpha^2 \frac {\sum_{z\notin \mathcal{Z}'}{E\|z\|^2}} {|\mathcal{Z}|-|\mathcal{Z}'|}
	\label{inequality:geovoting}
	\end{align}
where $C_\alpha :=\frac{2-2\alpha}{1-2\alpha}$ and $\alpha :=\frac{|\mathcal{Z}'|}{|\mathcal{Z}|}$.
	\label{lemma:geometric}
\end{Lemma}
\begin{proof}
With $z^*={{\rm geomed}}_{z \in \mathcal{Z}}\{z\}$ and $z \in \mathcal{Z}'$, it holds that  $\|z^*-z\|\ge\|z\|-\|z^*\|$; and for all $z \notin \mathcal{Z}'$, we have $\|z^*-z\|\ge\|z^*\|-\|z\|$. Then, summing up $\|z^*-z\|$ over all $z \in \mathcal{Z}$ yields
	\begin{align}
	\label{inequality:geometric-median-1}
	\sum_{z \in \mathcal{Z}} \|z^*-z\| \ge \sum_{z \in \mathcal{Z}} \|z\|+(|\mathcal{Z}|-2|\mathcal{Z}'|)\|z^*\|-2\sum_{z \notin \mathcal{Z}'}\|z\|.
	\end{align}
	According to the definition of geometric median, it holds that
	\begin{align}
	\sum_{z \in \mathcal{Z}} \|z^*-z\| = \inf_y \sum_{z \in \mathcal{Z}} \|y-z\| \leq \sum_{z \in \mathcal{Z}} \|z\|.
	\end{align}
	Combining the two inequalities, we arrive at
	\begin{align}
	\|z^*\|
	\le \frac{2 \sum_{z\notin \mathcal{Z}'} {\|z\|}}{|\mathcal{Z}|-2|\mathcal{Z}'|}
	= \frac {2|\mathcal{Z}|-2|\mathcal{Z}'|} {|\mathcal{Z}|-2|\mathcal{Z}'|} \frac {\sum_{z\notin \mathcal{Z}'} {\|z\|}} {|\mathcal{Z}|-|\mathcal{Z}'|}
	= C_\alpha \frac {\sum_{z\notin \mathcal{Z}'}{\|z\|}} {|\mathcal{Z}|-|\mathcal{Z}'|}
	\end{align}
and upon squaring both sides of the latter, we find
	\begin{align} \label{inequality:geovoting-statistic}
	\|z^*\|^2
	\le C_\alpha^2 \frac {( \sum_{z\notin \mathcal{Z}'}{\|z\|} )^2} { ( |\mathcal{Z}|-|\mathcal{Z}'| )^2}
	\le C_\alpha^2 \frac {\sum_{z\notin \mathcal{Z}'}{\|z\|^2}} {|\mathcal{Z}|-|\mathcal{Z}'|}.
	\end{align}
	Then taking expectations on both sides, yields \eqref{inequality:geovoting}, and completes the proof.
\end{proof}

With Lemma \ref{lemma:geometric}, the proof of Lemma \ref{lemma:geometric_error} is straightforward.

\begin{proof}
It follows readily from Lemma \ref{lemma:geometric} that
	\begin{align} \label{eq:temp001}
	E\|\underset{z\in \mathcal{Z}}{{\rm geomed}}\{z\}-\bar{z}\|^2
	=E\|\underset{z\in \mathcal{Z}}{{\rm geomed}}\{z-\bar{z}\}\|^2
	\le C_\alpha^2 \frac {\sum_{z\notin \mathcal{Z}'}{E\|z-\bar{z}\|^2}} {|\mathcal{Z}|-|\mathcal{Z}'|}.
	\end{align}
	Applying the inequality of
	\begin{align}
	E\|z-\bar{z}\|^2
	=\|z-Ez\|^2+ 2E\langle z-Ez, Ez-\bar z\rangle+\|Ez-\bar{z}\|^2
	=\|z-Ez\|^2+\|Ez-\bar{z}\|^2
	\end{align}
	to \eqref{eq:temp001}, yields
	\begin{align}
	E\|\underset{z\in \mathcal{Z}}{{\rm geomed}}\{z\}-\bar{z}\|^2
	\le C_\alpha^2 \frac {\sum_{z\notin \mathcal{Z}'}{E\|z-Ez\|^2}} {|\mathcal{Z}|-|\mathcal{Z}'|}
	+ C_\alpha^2 \frac {\sum_{z\notin \mathcal{Z}'}{E\|Ez-\bar{z}\|^2}} {|\mathcal{Z}|-|\mathcal{Z}'|}
	\end{align}
	which completes the proof.
\end{proof}

\section{Lemma \ref{lemma:geometric-epsilon} and Its Proof}

Since computing the accurate geometric median is difficult, we consider the $\epsilon$-approximate geometric median in this paper. The following lemma is the $\epsilon$-approximate counterpart of Lemma \ref{lemma:geometric}.

\begin{Lemma}
	\label{lemma:geometric-epsilon}
	Let $\{z : z \in \mathcal{Z}\}$ be a subset of random vectors distributed in a normed vector space. If $\mathcal{Z}'\subseteq \mathcal{Z}$ and $|\mathcal{Z}'|< \frac{|\mathcal{Z}|}{2}$, it holds that
	\begin{align}
	E \|z^*_\epsilon\|^2
	\le 2C_\alpha^2 \frac {\sum_{z\notin \mathcal{Z}'}{E\|z\|^2}} {|\mathcal{Z}|-|\mathcal{Z}'|}+\frac{2\epsilon^2}{(|\mathcal{Z}|-2|\mathcal{Z}'|)^2}
	\label{inequality:geovoting-epsilon-error-exp}
	\end{align}
    where $C_\alpha :=\frac{2-2\alpha}{1-2\alpha}$, $\alpha :=\frac{|\mathcal{Z}'|}{|\mathcal{Z}|}$, and $z^*_\epsilon$ is an $\epsilon$-approximate geometric median of $\mathcal{Z}$.
\end{Lemma}
\begin{proof}
	Because $z^*_\epsilon$ is an $\epsilon$-approximate geometric median, it follows that
	\begin{align}
	\sum_{z \in \mathcal{Z}} \|z^*_\epsilon-z\| \le \inf_y \sum_{z \in \mathcal{Z}} \|y-z\| +\epsilon \leq \sum_{z \in \mathcal{Z}} \|z\|+\epsilon.
	\end{align}
	Notice that \eqref{inequality:geometric-median-1} remains valid here. Hence, we have
	\begin{align}
	\|z^*_\epsilon\|
	\le C_\alpha \frac {\sum_{z\notin \mathcal{Z}'}{\|z\|}} {|\mathcal{Z}|-|\mathcal{Z}'|}+\frac\epsilon {|\mathcal{Z}|-2|\mathcal{Z}'|}.
	\label{inequality:gm-inequ-1}
	\end{align}
Squaring both sides of \eqref{inequality:gm-inequ-1}, leads to
	\begin{align}
	\|z^*_\epsilon\|^2
	&\le \left(C_\alpha \frac {\sum_{z\notin \mathcal{Z}'}{\|z\|}} {|\mathcal{Z}|-|\mathcal{Z}'|}+\frac\epsilon {|\mathcal{Z}|-2|\mathcal{Z}'|}\right)^2  \\
	&\le 2C_\alpha^2 \left( \frac {\sum_{z\notin \mathcal{Z}'}{\|z\|}} {|\mathcal{Z}|-|\mathcal{Z}'|}\right)^2+\frac{2\epsilon^2} {(|\mathcal{Z}|-2|\mathcal{Z}'|)^2} \\
	&\le 2C_\alpha^2 \frac {\sum_{z\notin \mathcal{Z}'}{\|z\|^2}} {|\mathcal{Z}|-|\mathcal{Z}'|}+\frac{2\epsilon^2} {(|\mathcal{Z}|-2|\mathcal{Z}'|)^2}.
	\end{align}
	Then taking expectations on both sides, yields \eqref{inequality:geovoting-epsilon-error-exp}, and completes the proof.
\end{proof}

\section{Lemma \ref{lemma:medfxexperror} and its proof}

As we have indicated in Section \ref{subsec:convergence}, the main challenge in the proof of Byrd-SAGA is that the geometric median of $\{m_i^k\}$ is a biased estimate of the gradient $f'(x^k)$. To handle the bias, the following lemma characterizes the error between an $\epsilon$-approximate geometric median of $\{m_w^k\}$ and $f'(x^k)$ per slot $k$.

\begin{Lemma}   \label{lemma:medfxexperror}
	Consider Byrd-SAGA with $\epsilon$-approximate geometric median aggregation. Under Assumptions \ref{assumption:muL} and \ref{assumption:outterVariance}, if the number of Byzantine attackers satisfies $B < \frac{W}{2}$, then an $\epsilon$-approximate geometric median of $\{m_w^k\}$, denoted by $z^*_\epsilon$, satisfies
	\begin{align}
	\label{equation:medfxexperror}
	E\|z^*_\epsilon-f'(x^k)\|^2\le
	2C_\alpha^2L^2S^k + 2C_\alpha^2\delta^2
	+\frac{2\epsilon^2}{(W-2B)^2}
	\end{align}
	where
	\begin{align}
	C_\alpha := \frac{2-2\alpha}{1-2\alpha} \quad \text{and} \quad \alpha := \frac{B}{W}
	\label{definition:C}
	\end{align}
    while $S^k$ is defined as
	\begin{align}
	S^k & := \avgHonest \avg{j}{J}\|x^k - \phi_{w,j}^k\|^2.
	\label{definition:Sk}
	\end{align}
\end{Lemma}

\begin{proof}
	We begin with upper bounding the mean-square error $E\|m_w^k - f_w'(x^k)\|^2$, where $w \notin \mathcal{B}$. Using the definition of $m_w^k$ in \eqref{definition:mk}, we have for any $w \notin \mathcal{B}$ that
	\begin{align}
	& E\|m_w^k-f'_w(x^k)\|^2 \label{equation:lemma2-1} \\
	= & E\|f_{w, i_w^k}'(x^k)-f_{w, i_w^k}'(\phi_{w,i_w^k}^k)+\avg{j}{J} f_{w, j}'(\phi_{w,j}^k)-f'_w(x^k)\|^2 \nonumber \\
	= & E\|f_{w, i_w^k}'(x^k) - f_{w, i_w^k}'(\phi_{w,i_w^k}^k)\|^2 - \|f_w'(x^k) -\avg{j}{J}f_{w, j}'(\phi_{w,j}^k)\|^2 \nonumber \\
	\le & E\|f_{w, i_w^k}'(x^k) - f_{w, i_w^k}'(\phi_{w,i_w^k}^k)\|^2  \nonumber\\
	\le & L^2 E\|x^k - \phi_{w,i_w^k}^k\|^2 \nonumber\\
	= & L^2 \avg{j}{J}\|x^k - \phi_{w,j}^k\|^2 \nonumber
	\end{align}
where the second equality is due to variance decomposition $E\|a-Ea\|^2=E\|a\|^2-\|Ea\|^2$ with $a = f_{w, i_w^k}'(x^k)-f_{w, i_w^k}'(\phi_{w,i_w^k}^k)$, and $Ea = f_w'(x^k) -\avg{j}{J}f_{w, j}'(\phi_{w,j}^k)$; while the last inequality comes from Assumption \ref{assumption:muL}.
	
To further upper bound the mean-square error $E\|m_w^k - f'(x^k)\|^2$, we have that
	\begin{align}
	& E\|m_w^k - f'(x^k)\|^2 \label{equation:lemma2-2} \\
	= & E\|m_w^k - f_w'(x^k) + f_w'(x^k) - f'(x^k)\|^2
	\nonumber \\
	= & E\|m_w^k - f_w'(x^k)\|^2
	+ 2E\left\langle m_w^k - f_w'(x^k),f_w'(x^k) - f'(x^k)\right\rangle
	+\|f_w'(x^k) - f'(x^k)\|^2
	\nonumber \\
	= & E\|m_w^k - f_w'(x^k)\|^2 + \|f_w'(x^k) - f'(x^k)\|^2
	\nonumber \\
    \le & L^2 \avg{j}{J}\|x^k - \phi_{w,j}^k\|^2 + \delta^2. \nonumber
	\end{align}
where the last inequality relies on \eqref{equation:lemma2-1} and Assumption \ref{assumption:outterVariance}.

Next, we will derive an upper bound on $E\|z^*_\epsilon-f'(x^k)\|^2$. According to \eqref{inequality:geovoting-epsilon-error-exp} in Lemma \ref{lemma:geometric-epsilon} and \eqref{equation:lemma2-2}, it holds that
	\begin{align}
	& E\|z^*_\epsilon-f'(x^k)\|^2 \nonumber \\
	\leq & 2 C_\alpha^2\frac {1} {W-B} \sum_{w\notin \mathcal{B}}E\|m_w^k-f'(x^k)\|^2+\frac{2\epsilon^2}{(W-2B)^2}
	\nonumber \\
	\leq & 2C_\alpha^2 \avgHonest
	\left( L^2 \avg{j}{J}\|x^k - \phi_{w,j}^k\|^2 + \delta^2
	\right)+\frac{2\epsilon^2}{(W-2B)^2}
	\label{equation:lemma2-4}
	\end{align}
which completes the proof.
\end{proof}

\section{Lemma \ref{lemma:recSk} and its proof}

In Lemma \ref{lemma:medfxexperror}, the upper bound of $E\|z^*_\epsilon-f'(x^k)\|^2$ contains a time-varying term $S^k$. The following lemma characterizes the evolution of $S^k$.

\begin{Lemma}   \label{lemma:recSk}
    Consider Byrd-SAGA with $\epsilon$-approximate geometric median aggregation. Under Assumptions \ref{assumption:muL}, it holds that
	\begin{align}
		\label{equation:recSk}
		ES^{k+1} \le
		4J\cdot E\|x^{k+1} - x^{k} + \gamma f'(x^k)\|^2
		+ 4J\gamma^2L^2 \|x^k-x^*\|^2
		+ (1-\frac 1 {J^2})S^k
	\end{align}
	where $S^k$ is defined in \eqref{definition:Sk}.
\end{Lemma}

\begin{proof}
For the expectation of $ES^{k+1}$, we have that
	\begin{align}
	\label{inequality:recSk1}
	ES^{k+1}=& E\left(
		\avgHonest \avg{j}{J}
		\|x^{k+1} - \phi_{w,j}^{k+1}\|^2
	\right) \nonumber\\
	\le&
	E\left(
		\avgHonest \avg{j}{J}
		\left(
			(1+\beta^{-1})\|x^{k+1} - x^{k}\|^2 + (1+\beta)\|x^{k}- \phi_{w,j}^{k+1}\|^2
		\right)
	\right) \nonumber\\
	=&
	(1+\beta^{-1})\cdot E\|x^{k+1} - x^{k}\|^2 +
	(1+\beta)\left(
	\avgHonest \avg{j}{J}
	 E\|x^{k}- \phi_{w,j}^{k+1}\|^2
	\right) \nonumber\\
	=& (1+\beta^{-1})\cdot E\|x^{k+1} - x^{k}\|^2
	+ (1+\beta)(1-\frac 1 {J})\avgHonest \avg{j}{J}\|x^{k} - \phi_{w,j}^{k}\|^2 \nonumber\\
	=& (1+\beta^{-1})\cdot E\|x^{k+1} - x^{k}\|^2
	+ (1+\beta)(1-\frac 1 {J})S^k
	\end{align}
where the inequality comes from $\|a+b\|^2\le (1+\beta^{-1})\|a\|^2+(1+\beta)\|b\|^2$ for any $\beta>0$, and the third equality holds because at slot $k$, honest worker $w$ uniformly at random chooses one out of $J$ data samples. For the chosen data sample $j$, $\phi_{w,j}^{k+1} = x^{k}$; otherwise, $\phi_{w,j}^{k+1} = \phi_{w,j}^{k}$.
	
	Using the fact that $f'(x^*)=0$, the first term in the right-hand side of \eqref{inequality:recSk1} can be bounded as
	\begin{align*}
    \label{inequality:recSk2-1}
	&\|x^{k+1} - x^{k}\|^2 \\
	= & \|x^{k+1} - x^{k} + \gamma f'(x^k) - \gamma f'(x^k) + \gamma f'(x^*)\|^2\\
	\le & 2\|x^{k+1} - x^{k} + \gamma f'(x^k)\|^2 + 2\gamma^2 \|f'(x^k) - f'(x^*)\|^2\\
	\le & 2\|x^{k+1} - x^{k} + \gamma f'(x^k)\|^2 + 2\gamma^2L^2 \|x^k-x^*\|^2
	\end{align*}
	where the first inequality comes from $\|a+b\|^2\le2\|a\|^2+2\|b\|^2$, and the last inequality comes from Assumption \ref{assumption:muL}.
	
	Substituting \eqref{inequality:recSk2-1} into \eqref{inequality:recSk1}, and choosing $\beta= 1/J$, we have
	\begin{align}
	ES^{k+1} \le & (1+J)\cdot E\|x^{k+1} - x^{k}\|^2
	+ (1-\frac 1 {J^2})  S^k  \\
	\le & 2J\cdot E\|x^{k+1} - x^{k}\|^2
	+ (1-\frac 1 {J^2})  S^k \nonumber\\
	\le & 4J\cdot E\|x^{k+1} - x^{k} + \gamma f'(x^k)\|^2
	+ 4J\gamma^2L^2 \|x^k-x^*\|^2
	+ (1-\frac 1 {J^2})S^k \nonumber
	\end{align}
	which completes the proof.
\end{proof}

\section{Proof of Theorem \ref{theorem:convergence}}

\begin{proof} Let $z^*_\epsilon$ be the $\epsilon$-approximate geometric median of $\{m_w^k\}$. We begin by manipulating $E\|x^{k+1}-x^{*}\|^2$ as
	\begin{align}
	\label{equation:theorem-converge-1}
	& E\|x^{k+1}-x^{*}\|^2  \\
	=& E\|x^{k} -\gamma f'(x^k) -x^{*} + x^{k+1}-x^{k}+\gamma f'(x^k)\|^2 \nonumber \\
	\le & \frac{1}{1-\eta} \|x^{k} -\gamma f'(x^k) -x^{*}\|^2
	+ \frac{1}{\eta} E \|x^{k+1}-x^{k}+\gamma f'(x^k)\|^2, \nonumber
	\end{align}
	where $0<\eta<1$, and the inequality comes from $\|a+b\|^2\le \frac{1}{\eta}\|a\|^2+\frac{1}{1-\eta}\|b\|^2$.
	
	To bound the first term in the right-hand side of \eqref{equation:theorem-converge-1}, we use that $f_{w,i_w^k}$ is $\mu$-strongly convex and has $L$-Lipschitz continuous gradients. Using also the fact that $f'(x^*)=0$, we obtain
	\begin{align}
	\label{equation:theorem-converge-2}
	& \|x^{k} -\gamma f'(x^k) -x^{*}\|^2 \\
	= & \|x^{k} -\gamma (f'(x^k)-f'(x^*)) -x^{*}\|^2 \nonumber \\
	= & \|x^{k}-x^{*}\|^2 - 2\gamma\langle f'(x^k)-f'(x^*), x^{k}-x^{*}\rangle + \gamma^2 \|f'(x^k)-f'(x^*)\|^2 \nonumber \\
	\le & \|x^{k}-x^{*}\|^2 - 2\gamma\mu\|x^{k}-x^{*}\|^2 + \gamma^2L^2\|x^{k}-x^{*}\|^2 \nonumber \\
	= & (1-2\gamma\mu+\gamma^2L^2) \|x^{k}-x^{*}\|^2. \nonumber
	\end{align}
    Here $\langle f'(x^k)-f'(x^*), x^{k}-x^{*}\rangle \geq \mu\|x^{k}-x^{*}\|^2$ because $f$ is $\mu$-strongly convex; see Theorem 2.1.9 in \cite{Nesterov2003IntroductoryLO}. Further, because $f$ has $L$-Lipschitz continuous gradients, it holds that $\|f'(x^k)-f'(x^*)\|^2 \leq L^2\|x^{k}-x^{*}\|^2$.
	
	Substituting \eqref{equation:theorem-converge-2} into \eqref{equation:theorem-converge-1} yields
	\begin{align}
	\label{equation:theorem-converge-xxxx}
	E\|x^{k+1}-x^{*}\|^2
	\le & \frac{1-2\gamma\mu+\gamma^2L^2}{1-\eta} \|x^{k} -x^{*}\|^2
	+ \frac{1}{\eta}E\|x^{k+1}-x^{k}+\gamma f'(x^k)\|^2.
	\end{align}
	With $\eta=\gamma\mu/2$, as long as
    \begin{align}
		\label{condition:gamma-1}
		\gamma^2L^2 \le \frac{\gamma\mu} {2}
	\end{align}
    it follows that
    \begin{align}
		\frac{1-2\gamma\mu+\gamma^2L^2}{1-\eta}\le 1-\gamma\mu. \nonumber
	\end{align}
	Therefore, \eqref{equation:theorem-converge-xxxx} can be rewritten as
	\begin{align}
	\label{equation:theorem-converge-xxxx2}
	E\|x^{k+1}-x^{*}\|^2
	\le & (1-\gamma\mu) \|x^{k} -x^{*}\|^2
	+ \frac{2}{\gamma\mu} E \|x^{k+1}-x^{k}+\gamma f'(x^k)\|^2.
	\end{align}
	
	Then, we construct a \emph{Lyapunov function} $T^k$ as
	\begin{align} \label{definition:Tk}
	T^k:=\|x^k-x^*\|^2+c  S^k
	\end{align}
	where $c$ is any positive constant. According to the definition in \eqref{definition:Sk}, we know $S^k$ is non-negative. Therefore, $T^k$ is also non-negative.
	
	Substituting \eqref{equation:recSk} and \eqref{equation:theorem-converge-xxxx2} into \eqref{definition:Tk}, it follows that
	\begin{align} \label{inequality:ETk+1:1}
	E T^{k+1} \le & (1-\gamma\mu + 4cJ\gamma^2L^2) \|x^{k} -x^{*}\|^2
	+ \left(\frac{2}{\gamma\mu} + 4cJ\right) E\|x^{k+1}-x^{k}+\gamma f'(x^k)\|^2 + (1-\frac 1 {J^2})cS^k.
	\end{align}
	According to Lemma \ref{lemma:medfxexperror}, the second term on the right-hand side \eqref{inequality:ETk+1:1} can be bounded as
	\begin{align}
	E \|x^{k+1}-x^{k}+\gamma f'(x^k)\|^2 = \gamma^2 E \|z_\epsilon^*-f'(x^k)\|^2
	\le \gamma^2
	\left(
		2C_\alpha^2L^2S^k + 2C_\alpha^2\delta^2
		+\frac{2\epsilon^2}{(W-2B)^2}
	\right).
	\end{align}
	Hence, we have
	\begin{align} \label{inequality:ETk+1:2}
	E T^{k+1}
	\le & (1-\gamma\mu + 4cJ\gamma^2L^2) \|x^{k} -x^{*}\|^2
	+ \left(
		\left(1-\frac 1 {J^2}\right)c
		+\left(\frac{2}{\gamma\mu} + 4cJ\right)2C_\alpha^2\gamma^2L^2
	\right)S^k \\
	&+ \gamma^2\left(\frac{2}{\gamma\mu} + 4cJ\right)
	\left(
	2C_\alpha^2\delta^2
	+\frac{2\epsilon^2}{(W-2B)^2}
	\right). \nonumber
	\end{align}
	
	If we constrain the step size $\gamma$ as
	\begin{align}
	\label{condition:gamma-2}
	4cJ\gamma^2L^2 \le \frac{\gamma\mu}{2}
	\end{align}
the coefficient in front of $\|x^{k} -x^{*}\|^2$ satisfies
    \begin{align*}
    1-\gamma\mu + 4cJ\gamma^2L^2 \le 1-\frac{\gamma\mu}{2},
    \end{align*}
	and the factor $\frac{2}{\gamma\mu} + 4cJ$ satisfies
	\begin{align*}
		\frac{2}{\gamma\mu} + 4cJ
		\le \frac{2}{\gamma\mu} + \frac{\mu}{2\gamma L^2}
		\le \frac{5}{2\gamma\mu}.
	\end{align*}
    Similarly, if $\gamma$ and $c$ are chosen such that
	\begin{align}
	\label{condition:gamma-3}
	\frac{\gamma\mu}{2} < \frac{1}{2J^2}
	\end{align}
    and
	\begin{align*}
	c = \frac{10J^2C_\alpha^2\gamma L^2}{\mu}
	\ge \frac{5C_\alpha^2\gamma L^2}{\mu(1/ {J^2}-\gamma\mu/2)}
	\end{align*}
	the coefficient in front of $S^k$ satisfies
	\begin{align*}
	\left(1-\frac 1 {J^2}\right)c
	+\left(\frac{2}{\gamma\mu} + 4cJ\right)2C_\alpha^2\gamma^2L^2
	\le\left(1-\frac 1 {J^2}\right)c
	+\frac{5}{\mu} C_\alpha^2\gamma L^2
	\le (1-\frac{\gamma\mu}{2})c.
	\end{align*}
	Therefore, \eqref{inequality:ETk+1:2} becomes
	\begin{align} \label{inequality:ETk+1}
	E T^{k+1}
	\le & (1-\frac{\gamma\mu}{2}) \|x^{k} -x^{*}\|^2
	+ (1-\frac{\gamma\mu}{2})c S^k
	+ \frac{5\gamma}{2\mu}
	\left(
		2C_\alpha^2\delta^2
		+\frac{2\epsilon^2}{(W-2B)^2}
	\right) \\
	= & (1-\frac{\gamma\mu}{2})T^k
	+ \frac{5\gamma}{\mu}
	\left(
		C_\alpha^2\delta^2
		+\frac{\epsilon^2}{(W-2B)^2}
	\right). \nonumber
	\end{align}
	
	For simplicity, let also
	\begin{align}
	\tilde{\Delta}_2 := \frac{5\gamma}{\mu}
	\left(
	C_\alpha^2\delta^2
	+\frac{\epsilon^2}{(W-2B)^2}
	\right).
	\end{align}
	Using telescopic cancellation on \eqref{inequality:ETk+1} from slot $1$ to slot $k$ , we arrive at
	\begin{align}
	E T^{k}  \leq (1- \frac{\gamma\mu}{2})^{k} \left[T^0-\frac{2}{\gamma\mu} \tilde{\Delta}_2\right]
	+ \frac{2}{\gamma\mu}\tilde{\Delta}_2.
	\end{align}
	Here and thereafter, the expectation is taken over $i_w^t$ for all workers $w \notin \mathcal{B}$ and slots $t \leq k-1$.
	
The definition of the \textit{Lyapunov function} in \eqref{definition:Tk}, implies that
	\begin{align}
	E\|x^k-x^*\|^2
	&\le E T^k
	\le (1- \frac{\gamma\mu}{2})^{k} \Delta_1
	+\Delta_2
	\label{inequality:temp}
	\end{align}
	where the constants $\Delta_1$ and $\Delta_2$ are defined as
	\begin{align}
	\Delta_1 := \|x^0-x^*\|^2 - \Delta_2
	\end{align}
	\begin{align}
	\Delta_2 := \frac{2}{\gamma \mu} \tilde{\Delta}_2 =
	\frac{10}{\mu^2}
	\left(
		C_\alpha^2\delta^2
		+\frac{\epsilon^2}{(W-2B)^2}
	\right).
	\end{align}
	
	In our derivation so far, the step size $\gamma$ must satisfy \eqref{condition:gamma-1}, \eqref{condition:gamma-2} and \eqref{condition:gamma-3}, meaning that
	\begin{align*}
		\gamma \le \min\left\{
			\frac{\mu}{2L^2},
			\frac{\mu}{4\sqrt{5}J^{3/2}C_\alpha L ^2},
			\frac{1}{J^2\mu}
		\right\}.
	\end{align*}
	Therefore, we simply choose
	\begin{align*}
	\gamma \le \frac{\mu}{4\sqrt{5}J^2C_\alpha L^2}
	\end{align*}
and the proof is complete.
\end{proof}

\section{Proof of Theorem \ref{theorem:convergenceSGD}}

Let $z^*_\epsilon$ denote the $\epsilon$-approximate geometric median of $\{m_w^k\}$, where $m^k = f_{w, i_w^k}'(x^k)$ for $w \notin\mathcal{B}$, and arbitrary otherwise. Similar to the proof of Theorem \ref{theorem:convergence}, we first derive an upper bound on $E\|x^{k+1}-x^*\|$. Inequality \eqref{equation:theorem-converge-xxxx2} is still true for Byzantine attack resilient SGD with $\gamma< \mu/(2L^2)$, and the only difference is that $E\|x^{k+1}-x^{k}+\gamma f'(x^k)\|^2$ becomes
\begin{align*}
&E\|x^{k+1}-x^{k}+\gamma f'(x^k)\|^2 \\
= & \gamma^2 E\|z^*_\epsilon-f'(x^k)\|^2 \\
\leq & \gamma^2\left(
	2 C_\alpha^2\frac {\sum_{w\notin \mathcal{B}}E\|f_{w,i_w^k}'(x^k)-f'(x^k)\|^2} {W-B}+\frac{2\epsilon^2}{(W-2B)^2}
\right)
\nonumber \\
= &
\gamma^2\left(
	2C_\alpha^2 \avgHonest
	\left(
		E\|f_{w,i_w^k}'(x^k) - f_{w}'(x^k)\|^2 + \|f_w'(x^k) - f'(x^k)\|^2
	\right)+\frac{2\epsilon^2}{(W-2B)^2}
\right)
\nonumber\\
\le&
\gamma^2\left(
	2C_\alpha^2\sigma^2
	+2C_\alpha^2\delta^2
	+\frac{2\epsilon^2}{(W-2B)^2}
\right)
\end{align*}
where the first inequality and the second equality are analogous to those in \eqref{equation:lemma2-4}, while the last inequality comes from Assumptions \ref{assumption:outterVariance} and \ref{assumption:innerVariance}. Therefore, for Byzantine attack resilient SGD, we have
\begin{align}
\label{inequality:ETk+1-SGD}
E\|x^{k+1}-x^{*}\|^2
\le & (1-\gamma\mu) \|x^{k} -x^{*}\|^2
+ \frac{2}{\gamma\mu}E\|x^{k+1}-x^{k}+\gamma f'(x^k)\|^2 \\
\le & (1-\gamma\mu) \|x^{k} -x^{*}\|^2 + \frac{4\gamma}{\mu}\left(
	C_\alpha^2\sigma^2
	+C_\alpha^2\delta^2
	+\frac{\epsilon^2}{(W-2B)^2}
\right). \nonumber
\end{align}
Here and thereafter, the expectation is taken over $i_w^t$ for all workers $w \notin \mathcal{B}$, and slots $t \leq k-1$.

Using telescopic cancellation on \eqref{inequality:ETk+1-SGD} from slot $1$ to slot $k$, we deduce that
\begin{align}
E \|x^{k+1}-x^{*}\|^2   \leq (1- \gamma\mu)^{k} \Delta_1'
+ \Delta_2'
\end{align}
where $\Delta_1'$ and $\Delta_2'$ are defined as
\begin{align}
\Delta_1' := \|x^0-x^*\|^2 - \Delta_2'
\end{align}
\begin{align}
\Delta_2' := \frac{4}{\mu^2}\left(
C_\alpha^2\sigma^2
+C_\alpha^2\delta^2
+\frac{\epsilon^2}{(W-2B)^2}
\right)
\end{align}
and the proof is complete.

\end{document}